\documentclass[final, 12pt]{colt2019} 

\title[Discrepancy, Coresets, and Sketches in Machine Learning]{Discrepancy, Coresets, and Sketches in Machine Learning}
\usepackage{times}

\usepackage{amsmath}
\usepackage{algorithm}
\usepackage[noend]{algorithmic}
\usepackage{xcolor}
\usepackage{amsfonts}
\usepackage{ amssymb }
\usepackage{algorithm2e}
\newtheorem{fact}[theorem]{Fact}
\usepackage{soul}

\newtheorem*{theorem*}{Theorem}

\newcommand{\ip}[1]{\left \langle #1 \right \rangle}
\newcommand{\R}{\mathbb{R}}
\newcommand{\E}{\mathbb{E}}
\newcommand{\eps}{\epsilon}
\newcommand{\F}{\mathcal{F}}
\newcommand{\X}{\mathcal{X}}
\newcommand{\Q}{\mathcal{Q}}
\newcommand{\floor}[1]{\left \lfloor #1 \right \rfloor}
\newcommand{\ceil}[1]{\left \lceil #1 \right \rceil}

\coltauthor{%
 \Name{Zohar Karnin} \Email{zkarnin@amazon.com}
 \AND
 \Name{Edo Liberty} \Email{libertye@amazon.com}
}

\begin{document}

\maketitle

\begin{abstract}
This paper defines the notion of class discrepancy for families of functions.  
It shows that low discrepancy classes admit small offline and streaming coresets. 
We provide general techniques for bounding the class discrepancy of machine learning problems. 
As corollaries of the general technique we bound the discrepancy (and therefore coreset complexity) of logistic regression, sigmoid activation loss, matrix covariance, kernel density and any analytic function of the dot product or the squared distance.
Our results prove the existence of $\eps$-approximation $O(\sqrt{d}/\eps)$ sized coresets for the above problems. 
This resolves the long-standing open problem regarding the coreset complexity of Gaussian kernel density estimation.  
We provide two more related but independent results. 
First, an exponential improvement of the widely used merge-and-reduce trick which gives improved streaming sketches for any low discrepancy problem.
Second, an extremely simple deterministic algorithm for finding low discrepancy sequences (and therefore coresets) for any positive semi-definite kernel. 
This paper establishes some explicit connections between class discrepancy, coreset complexity, learnability, and streaming algorithms. 
\end{abstract}

\paragraph{Acknowledgments:} The authors sincerely thank Nikhil Bansal, Nikhil Srivastava, Jeff Phillips, Wai Ming Tai, and Camron Musco 
for generously sharing their time and ideas.
They helped us uncover the usefulness of Banaszczyk's theorem for proving Lemma~\ref{uc}, compare to other results on coresets and discrepancy
(specifically on Gaussian Kernel Density estimation), and understand the connection to graph sparsification and matrix column subset selection results. 

\section{Introduction}
The study of coresets in optimization as a whole and in machine learning specifically has a long history. 
The basic setup is as follows.
Suppose you are trying to optimize an expression over a set of items, data points, or examples.
The optimization problem is difficult. 
Its running time dependence on the input set size is square, cubic, or even exponential.
As a result, there is a strong incentive to reduce the cardinality of that set. 
The goal is, therefore, to pinpoint a small subset of data items which approximates the entire input set with respect to the optimization at hand. 
Such small sets are called coresets. 
This idea is very general and applies to geometric properties of the data \cite{agarwal2005geometric}, 
clustering \cite{DBLP:conf/compgeom/Har-PeledK05} \cite{DBLP:conf/stoc/FeldmanL11}, classification \cite{DBLP:conf/ijcai/Har-PeledRZ07}, regression \cite{DBLP:conf/nips/MunteanuSSW18} 
machine learning \cite{bachem2017practical}, density estimation \cite{DBLP:journals/corr/abs-1802-01751}, and many other problems. 

Obtaining small coresets and understanding the coreset complexity (the size of the minimal coreset) of different problems is of significant theoretical and practical importance.
While some problems obviously do not admit small coresets, others do. 
There are several results that connect the simplicity of the measure and its coreset complexity.
In this manuscript, we focus solely on sums of functions applied to the input items.
That is, for $\{x_i,\ldots,x_n\} \subset \X$ we measure $F(q) = \sum_i f(x_i,q)$ for $q \in \Q$ which is either some model parameters or a query.
For example, one could consider the sum of sigmoid activation losses $F(q) = \sum_{i=1}^{n} 1/(1+ \exp(\ip{x_i,q}))$ and $x,q\in \R^d$.
Using Chernoff's inequality and a union bound already shows that sampling $\tilde O(d/\eps^2)$ items gives an $\eps n$ approximation to this sum.\footnote{Using $\tilde O(\cdot)$ to suppress poly-logarithmic terms.} 
In general, for families of VC dimension $v$, a sample of $(v+\log(1/\delta))/\eps^2$ suffices \cite{talagrand1994sharper}.
For logistic regression and many other problems $O(d/\eps)$ samples are enough due to fast rate generalization results \cite{van2015fast}.
For Gaussian kernel density, it is known that a sample size of $O(1/\eps^2)$ suffices independently of $d$  \cite{lopez2015towards}.
These results require different analyses and seem to stem from different mathematical underpinnings. 
This paper provides a general framework for obtaining and improving on these results. 

Rademacher complexity (see for example \cite{Bartlett:2003:RGC:944919.944944}) is a standard measure of generalization.
In other words, bounding the Rademacher complexity is a good way to upper bound the sample complexity. 
A sample is an instance of a coreset which is chosen i.i.d.\ from the data (or the underlying distribution).
A carefully selected coreset can, at least potentially, be better than a uniformly sampled one. 
It can be smaller and still give the same generalization power or give better generalization with the same number of data points. 
There are papers such as \cite{langberg2010universal, tolochinsky2018coresets} and references therein that tie the coreset size to the VC dimension of the function family and the average sensitivity of the dataset. 
These relationships come up as tools for constructing coresets rather than complexity measures aimed to characterize generalization ability. 
This paper defines the analog to Rademacher complexity that aims to characterize the coreset complexity, i.e., the generalization ability of the best possible coreset of a fixed size.
 
We show that our result applies to any analytic function of the dot product. 
These include Logistic Regression $F(q) =  \sum_i \log(1+\exp(\ip{y_ix_i,q}))$, Covariance or matrix approximation 
$F(q) =  \sum_i \ip{x_i,q}^2$, sigmoid activation loss $F(q) =  \sum_i 1/(1+\exp(\ip{y_ix_i,q}))$, linear regression $F(q) =  \sum_i (\ip{y_ix_i,q} -y_i)^2$ and many others. 
For all the aforementioned  $x,q \in R^d$ and $y_i \in \{-1,1\}$.
By bounding the class discrepancy of all such functions we prove the existence of coresets of size $O(\sqrt{d}/\eps)$ for all of them.

We note that while we obtain a universal {\it additive guarantee} it is often much harder to get a  {\it multiplicative guarantee}.
For Logistic regression, for example, a recent paper \cite{DBLP:journals/corr/abs-1805-08571} provides a coreset with a  {\it multiplicative guarantee} that is based on an average sensitivity property of the dataset. 
They provide a lower bound for the size of a multiplicative error coreset proving in particular that in general, it is not possible to achieve $m \ll n$. 
The coreset they build is of cardinality $m \approx \mu\sqrt{nd^3}/\eps^2$ where $\mu \geq 1$ is a complexity measure of the dataset.  
\cite{tolochinsky2018coresets} give a generic multiplicative coreset construction for any monotonic function with $\ell_2^2$ regularization. 
The dependence they get is $\tilde O(d/\eps^2)$ ignoring logarithmic factors. 
Additive approximation coresets are also studies in the $\eps$-approximation literature which is also related to the discrepancy of the problem \cite{DBLP:journals/corr/MustafaV17}. In~\cite{braverman2016new}, some connections are drawn between additive and multiplicative guarantees by providing a method to use an additive guarantee along with sensitivity scores in order to provide a multiplicative guarantee.\footnote{These methods might be combined with our results to obtain improved multiplicative guarantees, but this would not be a trivial result and we defer it to future research.}

We show that our result also applies to any analytic function of the squared distance.
A prime example of that is Gaussian kernel density estimation.
Kernel density estimation is a popular tool in data analysis aimed to estimate a continuous distribution with a finite set of points. 
Among other applications, this tool is used for outlier detection \cite{schubert2014generalized}, regression \cite{fan2018local}, and clustering \cite{rinaldo2010generalized}. 
A thorough survey could be found in \cite{silverman2018density}. 
Given a set of $n$ data points $\{x_1,\ldots,x_n\}$ and a query $q$, the Gaussian density estimate at point $q$ is given by $\sum_i K(x_i,q) = \sum_i e^{-\|x-q\|^2}$. 
Obtaining the smallest possible coreset for this problem has been open for several years.
The state-of-the-art is given by \cite{DBLP:journals/corr/abs-1802-01751} (see references within). 
They achieve coresets of size $O(\sqrt{d\log(1/\eps)}/\eps)$ where $d$ is the dimension of the original data points. 
Their result holds for any Lipchitz bounded positive semi-definite kernels. 
The result is constructive though it is polynomial rather than (quasi-)linear in the data size. 
The authors give an almost matching lower bound of $\sqrt{d}/\eps$ and pose an open question for closing the gap between the bounds. 
In this paper we resolve the open question by \cite{DBLP:journals/corr/abs-1802-01751} and prove that the coreset complexity of Gaussian kernel density is indeed $O(\sqrt{d}/\eps)$, matching the lower bound. 
In fact, we show that this is the coreset complexity for any bounded analytic function of the squared distance $f(x,q) = f(\|x-q\|^2)$.

In high dimensions $\sqrt{d}/\eps$ could be large. 
It is known (see \cite{lopez2015towards}, Theorem 1) that a uniform random sample of $\log(1/\delta)/\eps^2$ points gives a coreset w.p. $1-\delta$ for some kernel types. 
\cite{DBLP:journals/corr/abs-1802-01751} provide an algorithm based on the Frank-Wolf method that achieves a $1/\eps^2$ sized coreset. 
In section \ref{kde2} we provide (as a stand alone result) a very simple and deterministic algorithm for constructing coresets of size $1/\eps^2$ for any positive semi-definite kernel.  The worst-case coreset lower bound is $\Omega(1/\eps^2)$ which matches the coreset achieved by sampling. 
Yet, for real data the deterministic algorithm outperforms random sampling significantly (experiments not included in this manuscript).

\section{Class Discrepancy and Coreset Complexity}

In both machine learning and in streaming and sketching problems our goal is (often) to approximate sums or expectations of well-behaved functions.
Specifically, we need to approximate $\E_x f(x)$ or $\frac{1}{n}\sum_{i=1}^{n} f(x_i)$ for every $f\in \F$ where $\F$ is a family of functions and $x_i \in \X$ are either sampled training examples or an arbitrary set of stream items. 
Standard generalization results show that for a large enough value of $n$ the average approximates the mean if the complexity of $\F$ is bounded and the samples $x_i$ are drawn i.i.d.\ from an underlying distribution. We therefore focus on approximating the average, or rather, the sum $\sum_{i=1}^{n} f(x_i)$. 
For notational convenience, we use a parameter $q \in \Q$ to index into $\F$ explicitly. 
In other words, there is a bijective mapping between $\Q \equiv \F$ such that $f(x) \in \F$ iff there exists $q\in\Q$ such that $f(x,q)=f(x)$.
We keep using the two different functions $f:\X\rightarrow\R$ and $f:\X,\Q\rightarrow\R$ interchangeably. 
One should think about $q$ as either the model parameters or a query for the sketch. 

The goal is to produce a coreset. This is a small set $S\subset [n]$ and weights $w \in\R_+^{n}$ such that $\tilde F(q) = \sum_{i \in S} w_i f(x_i,q)$ approximates $F(q)$.
Approximation here means that $|\tilde F(q)  - F(q)| \le \eps n$ for all $q \in \mathcal Q$ simultaneously. 
There are more complicated formulations such as weak coresets which we will not touch upon in this manuscript. 
Generating a concise representation $\tilde F$ for $F$ allows one to optimize over $\tilde F$ instead of $F$ which is more efficient. 
Moreover, if the resulting coresets are mergeable, this could be done on separate streams without the need for communication or assuming randomness in the partitioning.

For bounded functions $f$, uniform sampling of $O(\log(1/\delta)/\eps^2)$ combined with a union bound over $|\mathcal Q|$ always provides a valid solution using $O(\log(|\mathcal Q|)/\eps^2)$ items. 
While $|\mathcal Q|$ is often infinite it can be replaced by a finite (albeit usually exponentially large) epsilon net $Q_\eps$. 
We present a mechanism for producing coresets which are much smaller than those achieved by sampling for a large class of problems in a unified manner.  
Moreover, our solutions create streaming algorithms with fully mergeable sketches. 
The size of the optimal coreset appears to be intimately tied to the class discrepancy properties of the associated functions.

\subsection{Class Discrepancy}
We begin by giving three equivalent definitions of complexity based on discrepancy for sets, functions, and function families. We will use all three interchangeably throughout the manuscript. 
Our notation is intentionally similar to the definition of the Rademacher complexity for reasons that will become clear later.

\begin{definition}[\bf Class Discrepancy] \label{cd1}
Let $A \subset \R^m$ and $\sigma \in \{-1,1\}^m$ the class discrepancy of $A$ is 
$
D_m(A)  = \min_\sigma \max_{a\in A} \left| \frac{1}{m} \sum_{i=1}^{m}  \sigma_i  a_i\right| 
$.
\end{definition}
\begin{definition}[\bf Class Discrepancy]\label{cd2}
Let $f:\X,\Q\rightarrow\R$ and $\sigma \in \{-1,1\}^m$. 
The class discrepancy of $f$ w.r.t.\ $\{x_1,\ldots,x_m\} \subset \X$ is 
$
D_m(f) =  \min_{\sigma} \max_{q \in \Q}  \left|\frac{1}{m} \sum_{i=1}^{m}\sigma_i f(x_i,q)\right|
$.
\end{definition}
\begin{definition}[\bf Class Discrepancy]\label{cd3}
Let $\F$ be a family of functions $f:\X\rightarrow\R$ and $\sigma \in \{-1,1\}^m$. 
The class discrepancy of $\F$ w.r.t.\ $\{x_1,\ldots,x_m\} \subset \X$ is 
$
D_m(\F) =  \min_{\sigma} \max_{f \in \F}  \left| \frac{1}{m} \sum_{i=1}^{m}\sigma_i f(x_i)\right|
$.
\end{definition}
The class discrepancy of $f$ or $\F$ without a reference a set $\{x_1,\ldots,x_m\}$ is the upper bound on any subset of $\X$ of size $m$.
Throughout the manuscript, we assume a bijective mapping between $\F$ and $\Q$. 
Specifically, any function in $\F$ can be written as $f_q$ and has a unique $q \in \Q$ such that $f_q(x)=f(x,q)$. 
In the context of machine learning, one should think about $f(x,q)$ as the loss associated with example $x$ and model parameters $q$.
The set $A$ should be thought of as the set of all possible induces loss vectors. 
Namely, $a\in A$ if there is a model $q$ such $a_i = f(x_i,q)$.

To understand our motivation, consider the following informal explanation of the Rademacher Complexity applied to ML problems. 
In PAC learning there exists a set of examples (often with labels).
We aim to find a regressor/classifier from a given family that suffers the least loss on the set. 
Having a low Rademacher complexity means that we can optimize over a sample of roughly half the examples at random (each w.p.\ $1/2$).
Low Rademacher complexity guaranties that, in expectation, twice the loss on the sample is roughly the same as the loss on the entire set.
This translates to a generalization bound. 
In other words, the Rademacher complexity gives a guarantee for the loss of coresets chosen uniformly at random.

Coming back to discrepancy. Having the ability to choose the signs arbitrarily lets us choose an advantageous subset of examples.
We can algorithmically choose to minimize the induced error and guarantee to have (roughly) the same performance on the entire set. 
This set is, in fact, a coreset. The class discrepancy of a problem helps us determine the obtainable coreset size. 
We will show in the following sections several examples for which a coreset can be significantly smaller than the random sample while maintaining the same guarantees. This will be done by showing that for a wide range of interesting problems in machine learning $D_m(\F) = o(R_m(\F))$.
This intuition is restated more explicitly in the next section.

\subsection{Coreset Complexity}
In this section, we point out a direct connection between coreset complexity and class discrepancy.
The connection is a simple application of the folklore argument know as the ``the halving trick".
For simplicity, in what follows we focus on functions $f$ whose range is $[0,1]$.

\begin{definition} [\bf Coreset Complexity] 
For a function $f:\X,\Q\rightarrow \R$ let $F(q) = \sum_{i=1}^{m} f(x_i,q)$ for any set $\{x_1,\ldots,x_m\} \subset \X$.
For a set $S \subset [m]$ let $\tilde F(q) = \sum_{i \in S}w_i f(x_i,q)$ for some $w\in\R_+^m$ which is independent of $q$.
The coreset complexity of $f$ is the size of the smallest set $S$ such that $\forall q \in \Q \; |F(q)  - \tilde F(q)| \le \eps m$.
\end{definition}
\noindent The following facts are true for the common cases where $D_m = O(c/m)$ or $D_m = O(c/\sqrt{m})$. Although they were previously known (see e.g., \cite{phillips2009small}, Theorem 1.1) we give their proof here for completeness.
\begin{fact}\label{fct:eps1}
Any function $f$ with class discrepancy $D_m = O(c/m)$ has coreset complexity of $O(c/\eps)$.
\end{fact}
\begin{fact}\label{fct:eps2}
Any function $f$ with class discrepancy $D_m = O(c/\sqrt{m})$ has coreset complexity $O(c^2/\eps^2)$.
\end{fact}
\begin{proof}
\noindent For a set of $n$ points $x_1,\ldots,x_n$ and arbitrary query $q$, consider the signed-sum error function $E(q) = \sum_{i=1}^{n} \sigma_i f(x_i,q)$ where $\sigma_i \in \{-1,1\}$.
Recalling $F(q) = \sum f(x_i,q)$, we consider $\tilde F_{+}(q) = F(q) + E(q)   = \sum_{i | \sigma_i=1} 2 f(x_i,q)$ and similarly $\tilde F_{-}(q) = F(q) - E(q)  =  \sum_{i | \sigma_i=-1} 2 f(x_i,q)$. 
We have that both $\tilde F_{+}(q)$ and $\tilde F_{-}(q)$ are approximations for $F(q)$ obtained by coresets of item-weight $2$. 
The error is at most $|\tilde F_{\pm}(q)- F(q)| =  |E(q)|$, and one of the coresets are of cardinality of at most $n/2$. 
The above is true for any choice of signs $\sigma$, specifically, for those minimizing $\max_q | E(q)|$.
By definition we can select signs such that $|E(q)| \le m D_n$. 

Naturally, one could iterate this process. Starting with $n$ items and ending with $m$.
Let $F_t$ denote the (unweighted) sum of functions $f$ after $t$ iterations and $n_t$ denote the cardinality of the coreset.\footnote{The expression $a=b \pm c$ means $|a-b| \leq c$}
\[
F = F_0 = 2F_1 \pm n_0D_{n_0} = 4F_{2} \pm 2n_1D_{n_1} \pm n_0D_{n_0} = \ldots = 2^T F_T \pm \sum_{t=0}^{T-1} 2^t n_t D_{n_t}
\]
Here $T$ stand for the total number of iterations. Let us analyze the error term.
Given $n_t \le n/2^{t} \approx m$ and the polynomial dependence of $D_m$ on $m$ we have 
$$\sum_{t=0}^{T-1} 2^t n_t D_{n_t} \le n \sum_{t=0}^{T-1} D_{n/2^t} = n\cdot O(D_{m}) . $$ 
Setting $m$ for which $D_{m} = \eps$ gets coresets with appropriate cardinalities and completes the proof.

\end{proof}

\begin{fact}
Class discrepancy bounds are tight asymptoticly for {\it unweighted} coreset complexity.
\end{fact}
\begin{proof}Taking for example the bound of Fact~\ref{fct:eps1}, if we can guarantee the existence of an unweighted coreset of size $c/\eps$, then for $m$ items a coreset of size $m/2$ provides a sign assignment with an error of $\eps=2c/m$, leading to an upper bound of $O(c/m)$ on the class discrepancy.
\end{proof}

The following it a straight forward fact which is provided mainly for convenience. 
It loosely says that optimizing models on coresets generalizes. In other words, ERM works as expected.
\begin{fact}
Let $f(x,q)$ be the loss suffered by model $q$ on example $x$. 
Let $R(q) = \frac{1}{n}F(q) = \frac{1}{n} \sum_{i=1}^n f(x_i,q)$ be the {\it empirical} risk associated with it.
Let $q^*$ denote the best empirical risk minimizer on the data ($q^* = \arg\min_q F(q)$).
Let $\tilde q$ be the minimizer of $q$ over an optimal weighted coreset of size $m$ ($\tilde q = \arg\min_q \tilde F(q)$).
Then $R(\tilde q) \le R(q^*) + O(D_m)$.
\end{fact}
\begin{proof}
This fact follows from the standard argument about empirical risk minimization.
\[
R(\tilde q) = \frac1n F(\tilde q)  \le  \frac{1}{n} \tilde F(\tilde q) + O(D_m) \le \frac{1}{n} \tilde F(q^*) + O(D_m) \le  \frac{1}{n}  F(q^*) + O(D_m)  = R(q^*)+ O(D_m)
\]
The first and last transitions are by definition. The second and fourth are by the approximation bounds above.
The third transition is due to the optimality of $\tilde q$ for $\tilde F$
\end{proof}

\subsection{Streaming Coreset Complexity}\label{sec:sketch}
We claim that low class discrepancy implies concise streaming mergeable coresets as well. 
\begin{definition} [\bf Streaming Coreset Algorithm] 
A streaming coreset algorithm for $f:\X,\Q \rightarrow \R$ receives and arbitrary set $\{x_1,\ldots,x_m\} \subset \X$ one item after the other.
At time $t \le m$, the algorithm maintains a subset $S_t \subset \{x_1,\ldots,x_t\}$ and uses at most $O(|S_t|)$ auxiliary memory. 
At the end of the stream, the algorithm must output $S$ and $w$ such that  $\forall q \in \Q \; |F(q)  - \tilde F(q)| \le \eps m$ where 
$F(q) = \sum_{i=1}^{m} f(x_i,q)$ and $\tilde F(q) = \sum_{i \in S}w_i f(x_i,q)$.
The size of the streaming coreset is $\max_t |S_t|$.
\end{definition}

\begin{definition} [\bf Streaming Coreset Complexity] 
The streaming coreset complexity for $f:\X,\Q \rightarrow \R$ is the minimal streaming coreset size among all possible 
streaming coreset algorithms for $f$.
\end{definition}

\noindent The following statements upper bound streaming coreset complexities for functions. We note that these bounds are only poly-logarithmically larger than their offline counterparts. 

\begin{theorem} \label{thm:streaming11}
Any function $f$ with class discrepancy $D_m(f) = O(c/m)$ has streaming coreset complexity of $O\left(c\log^2(\eps n/c)/\eps\right)$.
\end{theorem}

\begin{theorem} \label{thm:streaming21}
Any function $f$ with class discrepancy $D_m = O(c/\sqrt{m})$ has streaming coreset complexity of $O\left(c^2\log^3(\eps^2 n/c) /\eps^2\right)$.
\end{theorem}

Theorems~\ref{thm:streaming11} and~\ref{thm:streaming21} are achieved by deterministic algorithms. 
They could be thought of extensions of the MRL algorithm \cite{MRL} for streaming quantile sketching. 
Quantile sketching falls into this framework since it corresponds to $f(x,q) = 1$ if $x>q$ and $0$ else. The techniques of the above Theorems could also be associated with~\cite{matousek1995approximations}, providing a merge-reduce framework for additive coresets. 
More details and the proof of correctness are given in Appendix ~\ref{app:sketch proof}.

Recently, \cite{DBLP:conf/focs/KarninLL16} provided an improved (optimal) streaming quantile coreset algorithm by improving the merge-reduce technique in a way tailored to the quantile problem. In the typical merge-reduce framework, the algorithm is based on finding an $\eps$-coreset on subsets of size dependent on $\eps$ rather than on $n$. The novelty of \cite{DBLP:conf/focs/KarninLL16} is in suggesting a way to use different values of $\eps$ for these local coreset constructions. This ends up providing a randomized algorithm with no dependence on $n$ and doubly logarithmic dependence on the failure probability. 
%
Generalizing their construction requires more work and the main ideas are as follows. 
We argued above that $\tilde F_{+}$ and $\tilde F_{-}$ are both good approximations for $F$. 
We can also take $\tilde F_{\pm}$ which is $\tilde F_{+}$ or $\tilde F_{-}$ with equally probability. 
Clearly $|\tilde F_{\pm} - F| \le |E|$ as before. But now, $\E[\tilde F_{\pm}] = F$ as well. 
In the streaming algorithm, we apply this compaction (converting $F$ to $\tilde F_{\pm}$) many times to small subsets of items from the stream. This allows us to use concentration results to bound the overall error. So far, analogous ideas where used in \cite{DBLP:conf/focs/KarninLL16};
the main departure is that $\tilde F_{\pm}$ has half the support of $F$ only in expectation.

\begin{theorem} \label{thm:streaming12}
Any function $f$ with class discrepancy $D_m(f) = O(c/m)$ has streaming coreset complexity of $O\left(c\log^2\log(|Q_\eps|/\delta)/\eps\right)$.
$Q_\eps$ is an epsilon net for $f$ on $\Q$. 
The streaming coreset algorithm is randomized and fails with probability at most $\delta$.
\end{theorem}

\begin{theorem} \label{thm:streaming22}
Any function $f$ with class discrepancy $D_m(f) = O(c/\sqrt{m})$ has streaming coreset complexity of $O\left(c^2\log^3\log(|Q_\eps|/\delta) /\eps^2\right)$.
$Q_\eps$ is an epsilon net for $f$ on $\Q$. 
The streaming coreset algorithm is randomized and fails with probability at most $\delta$.
\end{theorem}

The set $Q_\eps$ is an $\eps$-net for $\Q$. 
It is a finite subset of $\Q$ such that for every $q \in \Q$ there exist some $\tilde{q} \in Q_\eps$ for which $\sup_{x \in \X} |f(q,x) - f(\tilde{q},x)| < \eps$. 
We note that the size $|Q_\eps|$ is often exponential in the problem parameters. 
Nevertheless, our dependence on the failure probability is \emph{doubly} logarithmic. 
This means the dependence on the problem parameters is still only polylogarithmic.
The above improves on the well-known merge-and-reduce tree construction by \cite{BENTLEY1980301}.
Moreover, it is likely that a uniform $\eps$-net for $\Q$ is not required for the sake of minimization (ERM on the final sketch).
See literature on weak coresets (e.g. \cite{Feldman:2007:PKC:1247069.1247072}) and concentration results based on doubling dimensions in classification/query space \cite{BSHOUTY2009323}. The refinement of the above results is left for future work.
 
\section{Class Discrepancy of Analytic Functions of Dot Products} \label{sec:analytic}

Now that we proved the usefulness of low class discrepancy, we move to upper bound it for common family functions. We provide a coreset suitable for analytical functions of the inner product $\ip{q,x}$ or squared Euclidean distance $\|q-x\|^2$. 
The idea is to find a set of signs that simultaneously balance $\ip{q,x}^k$ for all powers $k$ and unit vectors $q$.\footnote{We Assume that $\|x\|,\|q\| \leq 1$ for ease of presentation. As above our results extend to generic bounds on the radius of $q$}
By controlling all powers of $\ip{q,x}$ we control any sum of these powers. It follows that this coreset can be used to control, for example, the logistic loss function $L(q,x) = \log(1+\exp(\ip{q,x}))$, the gaussian Kernel $K(q,x) = \exp(-\lambda \|q-x\|^2)$, or the sigmoid activation loss $1/(1 + \exp(\ip{q,x}))$.

We start with some notation and trivial properties. 
For a vector $q \in \R^d$ let $q^{\otimes k}$ represent the $k$-dimensional tensor obtained from the outer product of $q$ with itself $k$ times. For a $k$ dimensional tensor with $d^k$ entries $X$ we consider the measure
$\|X\|_{T_k} = \max_{q \in \R^d, \|q\|=1} \left| \langle X, q^{\otimes k}\rangle \right|$.
\begin{fact}
$\|X\|_{T_k}$ is a norm
\end{fact}
\begin{proof}
We prove the claim directly from the definition of a norm.
Notice that for any $X \neq 0$, $\ip{X, q^{\otimes k}}$ is a non-zero polynomial in $q$. It follows that there must be $q$ for which its value is non-zero, meaning that $\|X\|_{T_k}=0$ iff $X=0$. For a scalar $a$, we clearly have by definition that
$\|aX\|_{T_k} = |a|\|X\|_{T_k}$.  Lastly, by the max definition we  have
$ \|X+Y\|_{T_k} =  \max_q \left| \langle X+Y, q^{\otimes k}\rangle \right| \leq 
\max_q \left| \langle X, q^{\otimes k}\rangle \right| + \max_q\left| \langle Y, q^{\otimes k}\rangle \right| = \|X\|_{T_k} + \|Y\|_{T_k}$
\end{proof}

We are now ready for the lemma controlling all powers of inner products simultaneously. 

\begin{lemma}\label{uc}
For any set of vectors $x_i \in \R^d$ with $\|x_i\| \leq 1$ there exist a set of signs $\sigma_i$ such that for all $k$ simultaneously $\left\| \sum_i \sigma_i x_i^{\otimes k} \right\|_{T_k} \le O(\sqrt{d k\log^{3}{k}})$ (the $3$ power of the term $\log(k)$ can be reduced to any constant power larger than $2$). 
\end{lemma}
\begin{proof}
The proof will use Banaszczyk's theorem \cite{Banaszczyk}. 
Let $\mathcal K$ be a convex body in Euclidean space with Gaussian measure at least 1/2 ($\Pr[g \in \mathcal K] \ge 1/2$ when $g$ is i.i.d.\ Gaussian).
Let $x_1,\ldots,x_n$ be vectors with $\|x_i\| \leq 1$. 
Then, there exist signs $\sigma$ such that $\sum \sigma_i x_i \in C \mathcal K$ for some constant $C$.

To use Banaszczyk's theorem we begin with defining our convex body.
Define the norm $\|\psi\|_T$ of a vector $\psi$ as follows. Look at the first $d$ coordinates of $\psi$ as a vector $\psi_1$, the next $d^2$ coordinates of $\psi$ as a matrix $\psi_2$ the next $d^3$ coordinates as a three tensor $\psi_3$ etc.
We define $\|\psi\|_T = \max_k \|\psi_k\|_{T_k} /\sqrt{\log(k)}$. 
Here, $\|\cdot\|_{T_k}$ is the special spectral norm defined in the beginning of the section.
The maximum over norms of subvectors is clearly a norm in itself, meaning that $\|\cdot \|_T$ is indeed a norm. It follows that  the set $\mathcal K  = \{\psi \; | \; \|\psi\|_T \le c\sqrt{d}\}$ is convex. 

We now need to show that the Gaussian measure of $\mathcal K$ is at least $1/2$. 
That is, with probability at least $1/2$ a vector of random Gaussian entrees $g$ belongs to $\mathcal K$.
Consider a random i.i.d.\ Gaussian Tensor $g_k \in \R^{d^k}$. 

A trivial modification of Theorem 1 from \cite{tomioka2014spectral} shows that $\Pr[\|g_k\|_{T_k} \ge c\sqrt{d\log(k)}] \le 1/10k^2$ for some constant $c$. The only change needed in the proof is the size of the epsilon net which changes from $(2\log(3/2)/k)^{kd}$ for \cite{tomioka2014spectral} to $(2\log(3/2)/k)^d$. The reason we require a net over a smaller space is due to us bounding the inner product with a rank one tensor rather than rank $k$. Union bounding on all values of $k$ we get $\sum_k 1/10k^2 \le 1/2$ which shows $g = [g_1, \operatorname{flat}(g_2), \operatorname{flat}(g_3), \ldots]$ belongs to $\mathcal K$ with probability at least $1/2$, where $\operatorname{flat}(g_k)$ is the flattening of the tensor into a one dimensional vector. 
We now define a mapping $\psi(x)$ of $x\in \R^d$ to a high dimensional space. 
$$\psi(x) = \left[x, \frac{\operatorname{flat}(x^{\otimes 2})}{\sqrt{2\log^2(2)}}, \frac{\operatorname{flat}(x^{\otimes 3})}{\sqrt{3\log^2(3)}}, \ldots,\frac{\operatorname{flat}(x^{\otimes k})}{\sqrt{k\log^2(k)}},\ldots \right]$$

\noindent Note that for $\|x\| \le 1$ we have $\|\psi(x)\|_2 = (\sum_k  1/k\log^2(k))^{1/2} = O(1)$.

We are now ready to apply Banaszczyk's theorem. 
There exist signs $\sigma_i$ such that $\psi  = \sum_i \sigma_i \psi(x_i) \in C \mathcal K$, meaning $\|\psi\|_T \leq C$.
Since $\psi_k = \sum_i \sigma_i x_i^{\otimes k}/\sqrt{k \log^2{k}}$ we get that 
$$\max_k \frac{\|\sum_i \sigma_i  x_i^{\otimes k}\|_{T_k}}{\sqrt{k \log^{3}(k)}} \le O\left( \sqrt{d} \right)$$
This concludes the proof of the statement.
\end{proof}

\begin{lemma} \label{lem:komlos anl}
Let $f$ be a function of the inner product $f(x,q) = f(\langle x,q\rangle)$ and let $f = \sum_k \alpha_k \langle x,q\rangle^k$ be its Taylor expansion. 
The class discrepancy of $f$ indexed by $\|q\| \leq 1$ is bounded by
\[
D_m = \min_\sigma \sum_i \sigma_i f(x_i,q) =O\left( \sqrt{d} \sum_k  |\alpha_k|\sqrt{k\log^3(k)}\right)
\]
For general $\|q\| \leq R$ we get
\[
D_m = \min_\sigma \sum_i \sigma_i f(x_i,q) =O\left( \sqrt{d} \sum_k  |\alpha_k| R^k \sqrt{ k\log^3(k)}\right)
\]
\end{lemma}
\begin{proof}
The proof follows from combining the above.
$$
\sum_i \sigma_i f(x_i,q) = \sum_k \alpha_k \sum_i \sigma_i \ip{ x_i,q}^k =  \sum_k \alpha_k  \ip{  \sum_i \sigma_i x_i^{\otimes k},q^{\otimes k}} \le $$
$$\sum_k |\alpha_k| \cdot \left\| \sum_i \sigma_i x_i^{\otimes k}\right\|_{T_k} \cdot \|q\|^k
$$
By Lemma \ref{uc} we can find signs $\sigma$ such that 
$\left\| \sum_i \sigma_i x_i^{\otimes k}\right\|_{T_k} \le c\sqrt{d k \log^3(k)}$. Substituting into the above, the lemma follows.
\end{proof}

\begin{theorem}\label{analitic1}
Let $f:\R\rightarrow\R$ be analytic. There exist a radius $R$ such that functions $f = f(\ip{q,x})$, indexed by $\|q\| \leq R$, have class discrepancy $O(\sqrt{d}/m)$. 
\end{theorem}
\begin{proof}
Recall that for analytic functions $f$ we have $\left| \frac{d^k f}{dz^k}(z) \right|  \leq C^{k+1} k! $
for some constant $C$. Considering the taylor expansion of $f$ near zero, for $R < 1/C$ the sum
$ \sum_k  |\alpha_k| R^k \sqrt{ k\log^3(k)} \leq C \sum_k  (CR)^k \sqrt{ k\log^3(k)}$
corresponding to Lemma~\ref{lem:komlos anl} converges to a constant. The result follows.
\end{proof}

The following two corollaries apply to the Logistic function and sigmoid activation loss function. They are easy to obtain by noticing the coefficients of the functions' Taylor expansion.

\begin{corollary}\label{cdlogistic}
The class discrepancy of the Logistic function $f(\ip{q,x}) = \log(1+\exp(\ip{q,x}))$ in dimension $d$, for $\|q\| \leq 1$ is $O(\sqrt{d}/m)$.
\end{corollary}
\begin{corollary}\label{cdsigmoid}
The class discrepancy of the sigmoid activation loss function $f(\ip{q,x}) = 1/(1+\exp(\ip{q,x}))$ in dimension $d$, for $\|q\| \leq 1$ is $O(\sqrt{d}/m)$.
\end{corollary}

\begin{corollary}
The class discrepancy of the covariance function $f(\ip{q,x}) = \ip{q,x}^2$ in dimension $d$, for $\|q\| \leq 1$ is $O(\sqrt{d}/m)$. This gives coresets for matrix column subset selection such that $\|XX^T - \tilde X \tilde X^T\| \le \eps n$ where $\tilde X$ contains only $O(\sqrt{d}/\eps)$ rescaled columns of the matrix $X$.
\end{corollary}

\begin{theorem} \label{thm:analytic2}
Let $f:\R\rightarrow\R$ be analytic. There exist a radius $R$ such that the function $f(\|x-q\|^2)$, indexed by $\|q\| \leq R$, has class discrepancy $O(\sqrt{d}/m)$. 
\end{theorem}
\begin{proof}
By transforming $x$ to $\tilde{x} = (1, \sqrt{2}x, \|x\|^2)$ and $q$ to $\tilde{q} = (\|q\|^2, -\sqrt{2}q, 1)$ we get $\ip{\tilde{x},\tilde{q}} = \|q-x\|^2$. Moreover, $\|q\| \le R$ gives $\|\tilde q\| \le R^2+1$. The result follows from applying Theorem~\ref{analitic1} to $f(\ip{ \tilde q, \tilde x}) = f(\|q-x\|^2)$.
\end{proof}

\begin{corollary}
For cases where $\|q-x\| \leq 1$ for all $q,x$, the class discrepancy of the Gaussian kernel $K(q,x) = \exp(-\gamma \|x-q\|^2)$ in dimension $d$ is $O(\gamma\exp(\gamma)\sqrt{d}/m)$.
\end{corollary} 
This improves upon the recent result of \cite{DBLP:journals/corr/abs-1802-01751} by proving the existence of $\eps$ approximation corsets of size $\sqrt{d}/\eps$ for Gaussian kernel density, in the case where $\gamma$ is constant. 
This also resolves the open problem raised by \cite{DBLP:journals/corr/abs-1802-01751} and matches their lower bound.   
For non-constant $\gamma$ assume w.l.o.g. $\|q-x\| \le 1$. The Taylor series of the Gaussian kernel $K$ exhibits $|\alpha_k| \le \gamma^k/k!$.
Plugging into the equation in the proof of Theorem~\ref{thm:analytic2} we get that the sum determining the constant is upper bounded by
$$\sum_{k=1}^\infty \gamma^{k}(k\log^3(k))^{1/2}/k! = O\left(\sum_{k=1}^\infty \gamma^{k}/(k-1)!\right) = O\left(\gamma \exp(\gamma)\right)$$

\subsection{Towards an Efficient Algorithm}\label{kde2}

From section \ref{sec:analytic} we know that the class discrepancy of the Gaussian kernel is $D_m = O(\sqrt{d}/m)$. 
Here, we provide a computationally efficient bound that can be achieved with a straightforward algorithm of complexity $O(m^2)$. Together with the results of Section~\ref{sec:sketch} this provides an efficient sketching algorithm for Kernel Density Estimation. 
In fact, we show that for any positive kernel $D_m = O(1/\sqrt{m})$. This bound is superior to that of the previous section for high dimensions $d > m$. 
More importantly, there is a very simple, intuitive, and deterministic algorithm for computing the signs $\sigma$. 
Given a collection of data points $X = \{x_1,\ldots, x_n\}$ in $\R^d$ the density function $f: \R^d \rightarrow \R$ of a point $q$ is defined as $ F(q) = \sum_{i=1}^{n} K(x_i,q) $.
Here, $K$ is any \emph{positive semi-definite kernel} function. The most frequent examples include
$$ K(x,q) = \exp(- \|x-q\|_2^2/\lambda^2),\;\; K(x,q) = \exp(- \|x-q\|/\lambda) ,\;\;  K(x,q) = (1+\|x-q\|/_2^2/\lambda^2)^{-1}$$
where $\lambda$ is a scaling parameter. For simplicity, we assume that $K(x,x) \leq 1$ for all data points. Notice that for any kernel based on the distance we have $K(x,x)=1$ exactly for all $x \in \R^d$.

\begin{algorithm}
\begin{algorithmic}
\STATE {\bf input:} Kernel function $K:(\R^d,\R^d)\rightarrow[0,1]$, points  $\{x_1,\ldots,x_m\}$
\STATE {\bf output:} $\sigma \in \{-1,1\}^m$ such that $\max_q |\sum_i \sigma_i K(x_i,q) | \le \sqrt{m}$
\STATE$\sigma_1 = 1$
\FOR {$i = 2,\ldots,m$} 
        \STATE $\sigma_i = -\operatorname{sign} (\sum_{j=1}^{i-1}\sigma_j  K(x_j, x_i))$
\ENDFOR 
\end{algorithmic}
\caption{Low discrepancy algorithm for positive semi-definite kernels}\label{kernelAlg}
\end{algorithm}

\begin{theorem} \label{thm:disc simple kernel}
Algorithm \ref{kernelAlg} achieves $\max_q |\sum_i \sigma_i K(x_i,q) | \le \sqrt{m}$.
\end{theorem}

\begin{proof}
For any positive semi-definite kernel $K$ there exist a mapping $\phi: \R^d \to {\cal V}$ to an inner product space $\cal V$ such that 
$ K(x,q) = \ip{\phi(x), \phi(q)} $.
Using this function $\phi$ our objective function becomes
\[
|\sum_{i=1}^m \sigma_i K(x_i,q)| = |\sum_{i=1}^m \sigma_i \ip{\phi(x_i), \phi(q)} | = \left| \ip{ \sum_{i=1}^m \sigma_i \phi(x_i), \phi(q)}\right| \leq  \|\phi(q)\| \cdot \left\|  \sum_{i=1}^m \sigma_i \phi(x_i) \right\| 
\]
Since $\|\phi(q)\| \leq 1$ we reduced the problem to bounding the norm of $ \sum_{i=1}^m \sigma_i \phi(x_i) $.
We show by induction on $i$ that 
$\left\| \sum_{j=1}^i \sigma_j \phi(x_j) \right\|^2 \le \sum_{j=1}^i \left\|\phi(x_j)\right\|^2 \leq i$.
This is trivially true for $i=1$ since $\|\phi(x)\| \leq 1$. 
Using our induction assumption we get
\begin{eqnarray*}
\left\| \sum_{j=1}^{i}\sigma_j \phi(x_j)\right\|^2 &=& \left\|\sum_{j=1}^{i-1}\sigma_j \phi(x_j)\right\|^2 + \|\phi(x_i)\|^2 + 2\ip{ \sum_{j=1}^{i-1}\sigma_j \phi(x_j), \sigma_i \phi(x_i)} \\
&\le& \sum_{j=1}^{i-1} \|\phi(x_j)\|^2 + \|\phi(x_i)\|^2 + 2\sigma_i \sum_{j=1}^{i-1}\sigma_j K(x_j, x_i)\\
&=& \sum_{j=1}^{i} \|\phi(x_j)\|^2 - 2\left| \sum_{j=1}^{i-1}\sigma_j K(x_j, x_i) \right| \le \sum_{j=1}^{i} \|\phi(x_j)\|^2 \\
\end{eqnarray*}
The first equality simply unpacks the squared vector norm, the second transition is due to the induction assumption and the last substitutes our choice of 
$\sigma$ (and $\operatorname{sign}(z)\cdot z =  |z|$). This completes the proof that $|\sum_{i=1}^m \sigma_i K(x_i,q)| \le \sqrt{m}$ for all $q$.
\end{proof}

Using the framework above provides a deterministic coreset construction for kernel density estimation of size $O(1/\eps^2)$ such that $\forall \;q\;\; |\tilde F(q) - F(q)| \le \eps n$. This matches and simplifies the results achieved by \cite{DBLP:conf/soda/PhillipsT18} and \cite{DBLP:journals/corr/abs-1802-01751}. Theorem~\ref{thm:streaming21} leads to a deterministic streaming algorithm with a memory complexity of $O(\log^3(\eps^2 n)/\eps^2)$. For $L$-Lipchitz kernels, meaning $K$ such that $|K(x,q+h) - K(x,q)|/\|h\| \leq L$ for all $h \neq 0$, Theorem~\ref{thm:streaming22} leads to a randomized streaming algorithm with a memory complexity of $O\left(\log^3 \left( d \log\left(RLn/\delta\eps\right) \right) / \eps^2 \right)$ that succeeds in finding a coreset with probability $1-\delta$. The parameter $R$ is the maximum norm of a query. The argument goes through a union bound over an $\eps/L$-net over vectors of norm at most $R$, the size of which is $(RL/\eps)^{O(d)}$.

\paragraph{Note} Theorem~\ref{thm:disc simple kernel} provides an upper bound of $\sqrt{m}$ for the sign discrepancy. 
This upper bound is tight since there exist sets of vectors in high dimensions that requires it. 
For data that lends itself to density estimation, however, one should expect input vectors to be clustered together.
In such cases, the algorithm above performs much better than the worst-case bound predicts. 
We leave it to future work to define properties of the data that ensure better guarantees for Algorithm~\ref{kernelAlg}.
\bibliography{density}

\appendix

\section{Proofs for Section~\ref{sec:sketch}, Sketching Coresets} \label{app:sketch proof}

The proofs of Theorems \ref{thm:streaming11}, \ref{thm:streaming12}, \ref{thm:streaming21}, and \ref{thm:streaming22} all use the basic concept of a compactor. A compactor consumes a stream of items and outputs another stream. 
The output stream contains at most half the items from the input stream with double the weight. 
It does so by keeping a buffer of a certain capacity $m$. When a new item is inserted into the compactor it is added to its buffer. 
If the buffer is full, a compaction operation takes place. 
The compaction takes the elements in the buffer $x_1,\ldots,x_m$ and finds a low discrepancy assignment $\sigma$ such that 
$\max_q |\sum_i \sigma_i f(x_i,q)| \leq m D_m$. 
Note that such a sequence is guaranteed to exist by the definition of the class discrepancy. For cases where an algorithm for finding this sequence $\sigma$ is not known, our result applies for the guarantee of the $\sigma$ sequence obtained by the algorithm. 
That is, if it is possible to obtain a bound of $D_m$ yet we can only find signs obtaining a bound of $\tilde{D}_m > D_m$, our results for the obtainable signs apply for $\tilde{D}_m$.
Given the sign vector $\sigma$, the compactor appends either $\{ x_i | \sigma_i = 1\}$ or  $\{ x_i | \sigma_i = -1\}$ to the output stream. 

Consider a stream of data points $x_1,\ldots,x_n$ and the output stream of a compactor $z_1,\ldots,z_{\tilde{n}}$. The error associated with the new stream w.r.t.\ a query $q$ is defined as
$$ \sum_{i=1}^n f(x_i,q) - 2\sum_j f(z_j,q) \ .$$
This is the difference between the value of $q$ on the original stream and the output stream. For a compactor we would like to bound both the length of the output stream, and the absolute value of its error.

\begin{lemma} \label{lem:det compactor}
A deterministic compactor output the smaller of the two sets  $\{ x_i | \sigma_i = 1\}$ or  $\{ x_i | \sigma_i = -1\}$. Given an input of length $n$, the output has at most $n/2$ items, and the error of the output stream is bounded in absolute value by $nD_m$
\end{lemma}
\begin{proof}
We note that the argument about the length is obvious, so we proceed to bound the error. Consider a single compaction operation done on $m$ vectors $x_1,\ldots,x_m$. For a query $q$, let $F(q)=\sum_{i=1}^m f(x_i,q)$ be the evaluation on the items of the buffer. Let $\tilde F_{+}$ denote the function evaluated on $\{ x_i | \sigma_i = 1\}_{i=1}^m$ (similarly $\tilde F_{-}$ defined for negative signs). Also, let $E(q) = \sum_{i=1}^m \sigma_i f(x_i,q)$ for the signs $\sigma$ computed by the algorithm above. We have that 
$$\tilde F_{+}(q) = \sum_{i ,\; \sigma_i=1} 2f(x_i, q) = \sum_{i} f(x_i, q) +  \sum_{i} \sigma_i f(x_i, q) = F(q) + E(q)$$
$$\tilde F_{-}(q) = \sum_{i ,\; \sigma_i=-1} 2f(x_i, q) = \sum_{i} f(x_i, q) - \sum_{i} \sigma_i f(x_i, q) = F(q) - E(q)$$
meaning that the error for the items of the single compaction is bounded by
$$|\tilde F_{\pm}(q) - F(q)| = |E(q)| \le \max_q |\sum_i \sigma_i f(x_i,q)| = mD_m$$
Summing over all $n/m$ compactions we get that the overall error is bounded, in absolute value, by $nD_m$.
\end{proof}

Lemma~\ref{lem:det compactor} alone already allow us to prove Theorems~\ref{thm:streaming11} and~\ref{thm:streaming21}. 
The algorithms are a direct extension the well know MRL algorithm \cite{MRL} for quantile sketching. 
Note that for quantiles, $f(x,q) = 1$ if $q > x$ and $0$ else. 
A low discrepancy sequence is achieved simply by sorting the values and assigning $\sigma_i = 1$ for all evenly positioned values in the sorted order and $\sigma_i=-1$ to the odd positions. The above gives class discrepancy of $1/m$ for quantile approximation.
Theorems~\ref{thm:streaming11} and~\ref{thm:streaming21} below generalize this algorithm to any low discrepancy class.

\paragraph{Theorem~\ref{thm:streaming11}}
For any function family $\F$ with a corresponding class discrepancy $D_m = O(c/m)$ there exists an fully-mergeable streaming coreset deterministic algorithm of size $O\left(c\log^2(\eps n/c)/\eps\right)$ whose error is at most $\eps n$.
\begin{proof}
Consider feeding the output of the first compactor into a second one etc. Specifically, we start with a single compactor and open a second once it produced any output, then open a third compactor once the second produced output, etc.
Number the compactors $0,\ldots,H$. The weight of items given to compactors $h$ have weight $w_h = 2^h$. The length of the input stream seen by compactor is $n_h \le n/2^h$.

Each compactor contributes at most $w_h n_h D_m \le n D_m$ error. Moreover since the $H-1$ layer had outputs, we must have $m \leq n_{H-1}$ and
$$  \log_2(m) \leq \log_2(n_{H-1}) \leq \log_2(n) - (H-1)$$
leading to a bound $H \le \floor{log_2(n/m)}+1$.
The total error is therefore $H n D_m \le O(log(n/m) n D_m)$. 
Setting $m \ge m_0 = O(c\log(\eps n/c)/\eps)$ and replacing $D_m = c/m$ we get that the error is at most $O(log(n/m) n D_m) \le \eps n$.
Since we have $H = O(log(\eps n/c))$ compactors the overall space complexity is $O(c\log^2(\eps n/c)/\eps)$. 
\end{proof}

\paragraph{Theorem~\ref{thm:streaming21}}
For any function family $\F$ with a corresponding class discrepancy $D_m = O(c/\sqrt{m})$ there exists a fully-mergeable streaming coreset deterministic algorithm of size $O\left(c^2\log^3(\eps^2 n/c) /\eps^2\right)$ whose error is at most $\eps n$. 
\begin{proof}
The proof is identical to the one above except for the variable setting of
Setting $m \ge m_0 = O(c^2\log^2(\eps^2 n/c^2)/\eps^2)$ and replacing $D_m = c/\sqrt{m}$. 
We get that the error is at most $O(log(n/m) n D_m) \le \eps n$. 
Since we have $H = O(log(\eps^2 n/c^2))$ such compactors the overall space complexity is $O\left(c^2\log^3(\eps^2 n/c^2) /\eps^2\right)$. 
\end{proof}

We proceed to prove Theorem~\ref{thm:streaming22}. To understand the motivation consider first an easier setting where the overall stream length $n$ is known to us in advance. Since $|f(x,q)| \leq 1$, standard concentration bounds will show that by sampling each item w.p.\ $\log(1/\delta)/n\eps^2$ we get an output stream of length $\log(1/\delta)/\eps^2$, that for any fixed query $q$, with probability at least $1-\delta$ suffers an error of $\eps n$ for that query. We can feed this output stream into a deterministic sketch, and given that the input length for the deterministic sketch is $\log(1/\delta)/\eps^2$, Theorem~\ref{thm:streaming21} leads to the required guarantee.

Because we do not know the stream length in advance, we operate as in the deterministic case with compactors. The difference will be that each compactor will keep a count of how many items it has seen. Once a compactor observed more than $\tilde{n} = O(\log(1/\delta)/\eps^2)$ items, it will no longer use a buffer of size $m$ but rather a buffer of size 2. For every two items observed it will output one of them uniformly at random. It is easy to see that a sequence of such compactors can, in fact, be implemented with $O(1)$ memory via reservoir sampling. 
The memory of this process is therefore identical, at least asymptotically, to the above.

\paragraph{Theorem~\ref{thm:streaming22}}
For any function family $\F$ with a corresponding class discrepancy $D_m = O(c/\sqrt{m})$ there exists a fully-mergeable streaming coreset randomized algorithm of size $O\left(c^2\log^3\log(n/\delta) /\eps^2\right)$ whose error for any fixed function $f \in \F$ is at most $\eps n$ with probability at least $1-\delta$. 
\begin{proof}
As in the deterministic setting we maintain a sequence of compactors of levels $h=0,\ldots,H$. Notice that the value of $H$ is increasing as the stream grows longer. Recall that a compactor of level $h$ observes elements of weight $2^h$ and outputs elements of weight $2^{h+1}$. As before we use a buffer of $m$ and get that $H \leq  \floor{log_2(n/m)}+1$.
The difference is that for a compactor of level $h$, once $h \leq H' = H - \log(\tilde{n}/m)$, where $\tilde{n} = O(\log(1/\delta)/\eps^2)$ with a constant in the $O()$ term that will be determined later, we change the mode of operation for this compactor. Notice that the requirement for $h$ ensures that the number of items observed by the $h$'th compactor is at least $n_h \geq \tilde{n}$. Rather than using a buffer of size $m$ the compactor uses a buffer of size 2 and for every two observed items it outputs one of them uniformly at random.

To analyze the memory requirement, notice that the compactors of levels $h=0,\ldots,H'$ are in fact performing reservoir sampling for every $2^{H'+1}$ items, meaning that they can be implemented in $O(1)$ memory. This means that the overall memory requirement is $O(m\log(\log(1/\delta)/m\eps^2))$; for $m \geq 1/\eps^2$ this is $O(m\log\log(1/\delta))$. 

We continue to bound the error. For the top compactors of level $h=H'+1,\ldots,H$ we get as in the deterministic case that the error for each is $n D_m$. Since we will use $m \geq 1/\eps^2$ we get that the error for all top compactors is $O(nD_m\log\log(1/\delta))$.
Consider now a compactor of level $h \leq H'$. For the first $\tilde{n}$ items it observed, the error is bounded by $2^h\tilde{n}D_m \leq 2^{h-H'}nD_m$. Fix a query $q$; for the items following the first $\tilde{n}$ items the compactor is operating in the sampling mode. For every pair, the associated error w.r.t $q$ is a random variable, of mean zero and absolute value of at most $w_h=2^{h+1}$. There are $(n_h-\tilde{n})/2 \leq n_h$ such pairs and the overall error w.r.t.\ $q$ is the sum of these independent random variables. Chernoff bound implies that with probability $1-\delta$, the overall error is bounded by $E_h = O(w_h\sqrt{n_h\log(1/\delta)})$. Since $n_h \geq \tilde{n}2^{H'-h} = O(2^{H'-h}\log(1/\delta)/\eps^2)$ we get that
$$ E_h = O\left(2^h n_h\frac{\eps}{2^{(H'-h)/2}}\right) = O(\eps n 2^{(h-H')/2}) $$
We get that the sum of errors associated with the compactors of level $h=0,\ldots,H'$ form a geometric sequence dominated by the error of the $H'$ compactor, which is in turn $O(\eps n)$. For proper constants in $\tilde{n}$ we get a bound of $\eps n/2$ for the bottom compactors. For a budget of $m = \Omega(c^2\log^2(\log(1/\delta))/\eps^2)$ for the buffers of the top compactors we guarantee an overall error of $\eps n/2$ for the top compactors. 

To conclude, we get an error of $\eps n$ w.p. $1-\delta$ for any fixed $q$ with a memory budget of
$$O(m\log\log(1/\delta)) = O\left(c^2\log^3(\log(1/\delta))/\eps^2\right)$$
as required.
\end{proof}

We are now ready for the proof of Theorem~\ref{thm:streaming12}. Here we extend the idea of \cite{DBLP:conf/focs/KarninLL16} applied for quantiles to general coresets. To explain the high-level idea consider again the easier setting where we know $n$, the length of the stream in advance. As in the $D_m=c/\sqrt{m}$ case, we will split the compactors into the top $\log\log(1/\delta)$ ones acting deterministically and bottom compactors yielding random outputs. The issue comes from the choice of $m$. To handle the error of the top compactors it suffices to set $m=c/\eps \ll 1/\eps^2$. The fact that $m \ll 1/\eps^2$ means that the top random compactors observe a stream that is shorter than before and having a buffer of size 2 will result in a large error. We can mitigate this by adding $\log(1/\eps)$ more deterministic compactors and replace the $\log\log(1/\delta)^2$ term in the memory requirement with $\left(\log\log(1/\delta)/\eps\right)^2$. If $\log(1/\delta) \gg 1/\eps$ then this is a good solution. However, for cases where $\eps$ is small we can avoid the $\log(1/\eps)$ term altogether. To do that, the random compactors will not have a buffer of size 2, but a buffer size of $m_h$ depending on their level. Specifically the sequence of $m_h$ starting from the top random level $h= H-\log\log(1/\delta)$ and ending with $h=0$ is exponentially decreasing until hitting the minimal buffer size of $2$.

The memory requirement is now $O(m)$ and a careful analysis of the error will lead to an $\eps n$ term coming from the bottom layers. One subtle issue we will need to take into account is that for random compactors with budget $m_h>2$ the output stream is only half as long as the input stream in expectation. Luckily, the output stream length is sharply concentrated around its mean so a union bound can ensure that w.p. $1-\delta$ the output stream is not much longer than its expectation.

\paragraph{Theorem~\ref{thm:streaming12}} 
Any function $f$ with class discrepancy $D_m(f) = O(c/m)$ has streaming coreset complexity of $O\left(c\log^2\log(|Q_\eps|/\delta)/\eps\right)$.
$Q_\eps$ is an epsilon net for $f$ on $\Q$. 
The streaming coreset algorithm is randomized and fails with probability at most $\delta$.
\begin{proof}
We start by describing the algorithm, from the perspective of a compactor of level $h$. The compactor observes an input stream of items with weight $2^h$ and outputs a stream of weight $2^{h+1}$. When created the compactor has a budget of $m_h=m$. Once it outputs items to an output stream for the first time, a new compactor of level $h+1$ is created. We keep track of $H$, the level of the top compactor, that did not yet output any items. When $H$ is updated, compactors of level $h<H$ might restrict their budget. Specifically, for some $H'=H-O(\log\log(n/\delta))$ where we set the constant of the $O()$ term later, a compactor of level $h \leq H'$ sets its buffer size to
$$ m_h = \max\left\{2, \ceil{(2/3)^{h-H'}m} \right\} $$
compactors of level $h >H'$ have a buffer size of $m$. We note that although $n$ is present in the definition of $H'$ we can use a crude upper bound. Given that the dependence is doubly logarithmic the upper bound can be extremely crude. Furthermore, $\delta$ is typically set to be exponentially small, so we ignore this issue.

Compactors of level $h >H'$ act in a deterministic manner. Namely, once the buffer is full with items $x_1,\ldots,x_m$ we find the sign assignment $\sigma$ giving $\left|\max_q \sum \sigma_i f(x_i, q)\right| < m_hD_{m_h}=mD_m$ and output the smallest of the sets $X_+=\{x_i | \sigma_i > 0 \}$, $X_- = \{x_i | \sigma_i < 0 \}$. Compactors of level $h \leq H'$ act in a random manner; they output either the items of $X_-$ or $X_+$ with equal probability. When the stream is finished the coreset consists of all the items in the buffers, along with their corresponding weight.

Let's begin by analyzing the memory complexity of the algorithm. The top layers each require a buffer of size $m$, and there are $\log\log(n/\delta)$ such buffers. It follows that they require $O(\log\log(n/\delta)m)$ memory. The bottom layers are exponentially decreasing until hitting $m_h=2$. All layers with $m_h=2$ are stacked in a consecutive way so they are in fact doing reservoir sampling and can be implemented with $O(1)$ memory. The layers with $m_h>2$ are have exponentially growing weights ending at $m$, so the overall memory they require is $O(m)$. Concluding, the overall memory requirement is $O(\log\log(n/\delta)m)$.

We are now ready to bound the error, starting with the bottom layers. Fix a query $q$. For a layer $h$ we will provide a high probability bound to both $E_h(q)$, the error associated to its output stream and the length of the output stream. Let $n_h$ be the overall number of items layer $h$ observes. Let $m_h$ be the buffer size of level $h$ at the end of the stream. Since having a larger buffer size only improves the error bound, we analyze the error as if the budget was set as $m_h$ to begin with.

With the assumption of all compactions being done with a buffer of size $m_h$, the number of compactions is $n_h/m_h$ and the error associated with each compaction is a zero mean random variable, with an absolute value of $2^h m_h D_{m_h}$. The overall error $E_h(q)$ is the sum of these independent random variables. It follows from Chernoff-Hoeffding bound that for any $\eps_h > 0$,
\begin{equation} \label{eq:err low1}
\Pr\left[ E_h(q) > 2^h m_h D_{m_h} \eps_h n_h \right] = \exp \left( -\Omega\left( \eps_h^2 n_h m_h \right) \right)
\end{equation}

For a bound on the output length we will analyze the behavior of the compactor with the assumption that all compactions are done to $m$ elements. This is not the case but an upper bound for this scenario also bounds the scenario where $m_h$ is decreasing with time. Every compaction outputs a random number of items between $0$ and $m$, with an expected value of $m/2$. Again, using Chernoff-Hoeffding we get
\begin{equation} \label{eq:outlen}
\Pr\left[ n_{h+1} > n_h(1/2+1/\log(n)) \right] = \exp \left( -\Omega\left( \frac{n_h}{m \log^2(n)} \right) \right) 
\end{equation}
To bound this expression we derive a lower bounding on $n_h$. Notice that the compactors of levels $H'+1,\ldots,H$ are acting in a deterministic manner meaning that 
$$n_h \geq n_{H'} \geq 2^{H-H'-1}n_{H-1} \geq 2^{H-H'-1}m = \Omega(\log^2(n)\log(\log(n)/\delta)m)$$
where the constant in the $\Omega$ term can be controlled via constant defining $H'$. Plugging into Equation~\eqref{eq:outlen} leads to
\begin{equation*} 
\Pr\left[ n_{h+1} > n_h(1/2+1/\log(n)) \right] \leq \delta/2(\log_2(n)+3)
\end{equation*}
A union bound over $h=0,\ldots,\log_2(n)+2$ indicates that w.p. $1-\delta/2$, $n_h \leq 3n/2^h$ for all mentioned $h$ values. In particular this means that $H \leq \log_2(n)+2$ meaning that 
\begin{equation} \label{eq:outlen2}
\Pr\left[ \forall h, n_h \leq 3n/2^h \right] \geq 1-\delta/2
\end{equation}

We can now plug the upper bound for $n_h$ to Equation~\ref{eq:err low1} and achieve 
\begin{equation} \label{eq:err low2}
\Pr\left[ E_h(q) > m_h D_{m_h} \eps_h n \right] = \exp \left( -\Omega\left( \frac{\eps_h^2 m_h n}{2^h} \right) \right)
\end{equation}
Recall that $2^{H-H'-1} = \Omega(\log^2(n)\log(\log(n)/\delta))$ and $n \geq 2^{H-1}m$. Combining the two leads to $n = \Omega(2^{H'}\log(n/\delta)m)$. Now, since $m_h \geq (2/3)^{H'-h}m$ we have that
$$
\frac{\eps_h^2 n m_h}{2^h} = \Omega\left( (2/3)^{H'-h} \frac{\eps_h^2 2^{H'} \log(n/\delta) m^2}{2^h} \right) = \Omega\left( (4/3)^{H'-h} \eps_h^2 \log(n/\delta) m^2 \right)
$$
Plugging this into Equation~\eqref{eq:err low2}, with $\eps_h = (3/4)^{h-H'}/m$ and using $m_h D_{m_h} \leq c$ we get
\begin{equation} \label{eq:err low3}
\Pr\left[ E_h(q) > \frac{(3/4)^{h-H'} c}{m} n \right]  \leq \delta/2n
\end{equation}
Since $H' < n$ we get that w.p.\ $1-\delta/2$
$$ \sum_{h=1}^H E_h(q) \leq (4c/m) n $$

Concluding the analysis for the bottom $H'$ layers, w.p.\ at least $1-\delta$ their error is $(4c/m) n$ and the output stream of the $H'$ compactor outputs at most $3n/2^{H'+1}$ items, each having a weight of $2^{H'+1}$. With the length of the output stream we use the fact that the top layers are deterministic and can apply Lemma~\ref{lem:det compactor} to bound their error of each of these layers by
$$ 3D_m n \leq (3c/m) n $$
Since there are $O(\log\log(n/\delta))$ such layers, we conclude that for $m=\Omega(\log\log(n/\delta) c/\eps)$ with appropriate constant it holds for a fixed $q$, w.p.\ at least $1-\delta$ that the overall error of the sketch is bounded by $\eps n$. The resulting memory requirement $O(\log^2\log(n/\delta) c/\eps)$, as claimed.
\end{proof}

%

\end{document}